\documentclass[11pt]{article}

\usepackage[margin=1in]{geometry}
\usepackage{amsmath,amssymb,amsthm,mathtools}
\usepackage{microtype}
\usepackage{xcolor}
\usepackage[colorlinks=true,linkcolor=blue,citecolor=blue,urlcolor=blue]{hyperref}

\usepackage[nameinlink,capitalize,noabbrev]{cleveref}
\usepackage{algorithm}
\usepackage[noend]{algpseudocode}

\newtheorem{theorem}{Theorem}
\newtheorem{lemma}{Lemma}

\newtheorem{corollary}{Corollary}

\newtheorem{example}{Example}
\newtheorem{definition}{Definition}
\crefname{equation}{}{Equations}
\DeclareMathOperator*{\argmax}{arg\,max}

\newcommand{\X}{\mathcal X}

\newcommand{\E}{\mathbb E}
\newcommand{\R}{\mathbb R}
\newcommand{\KL}{\mathrm{KL}}

\newcommand{\supp}{\operatorname{supp}}
\newcommand{\AltReg}{\operatorname{Reg}_{\mathrm{Alt}}}
\newcommand{\dd}{\,\mathrm d}
\newcommand{\ARvec}[1]{\boldsymbol{#1}}
\newcommand{\ARzero}{\ARvec{0}}
\newcommand{\one}{\ARvec{1}}
\newcommand{\logp}{\log_{+}}

\newcommand{\reg}{\mathrm{Reg}}

\def\+#1{\mathcal{#1}}
\def\-#1{\mathbb{#1}}

\usepackage{natbib}
\let\oldcitet\citet
\let\oldcitep\citep
\renewcommand{\citet}[1]{\oldcitet*{#1}}
\renewcommand{\citep}[1]{\oldcitep*{#1}}

\title{Optimal Alternating Regret for Online Learning and Games\thanks{Authors are alphabetically ordered.}}
\author{
Yixin Tao\thanks{ITCS, Key Laboratory of Interdisciplinary Research of Computation and Economics, Shanghai University of Finance and Economics. Email: \texttt{taoyixin@mail.shufe.edu.cn}} \\
\and
Weiqiang Zheng\thanks{Yale University. Email: \texttt{weiqiang.zheng@yale.edu}}
}
\date{}

\begin{document}
\maketitle

\begin{abstract}
We settle the minimax-optimal alternating regret, a regret notion motivated by alternating learning dynamics in games, for both online linear optimization (OLO) and online convex optimization (OCO).

For OLO over the probability simplex $\Delta_d$, we give an algorithm with $O(\log d)$ alternating regret that remains a constant for any time horizon $T$, and a matching lower bound. Our constant regret bound significantly improves previous results with $O(\log ^{2/3}d \cdot T^{1/3})$ regret [Cevher, Cutkosky, Kavis, Piliouras, Skoulakis, Viano, NeurIPS 2023, Hait, Li, Luo, Zhang, COLT 2025]. As a result, we obtain alternating learning dynamics with $O(\log d /T)$ convergence to Nash equilibria in two-player zero-sum games and $O(\log d /T)$ convergence to coarse correlated equilibria in two-player general-sum games. This is the first uncoupled learning dynamics with $O(1/T)$ convergence to CCE in two-player general-sum games, while all prior works suffer additional $\log T$ factors.

For general OCO over a $d$-dimensional compact convex set, we give an algorithm with $O(d\log (1+T/d))$ alternating regret, improving the previous best of $\widetilde{O}(d^{2/3}T^{1/3})$. We also prove a matching lower bound of $\Omega(d\log (1+T/d))$, showing that the $\Omega(\log T)$ factor is unavoidable.

\end{abstract}

\tableofcontents

\section{Introduction}
We consider the standard online convex optimization (OCO) problem~\citep{zinkevich2003online, hazan2016introduction, orabona2019modern}. In each iteration $t \in [T]$, the learner first chooses an action $\ARvec{x}_t \in \+X$ from a $d$-dimensional convex set $\+X$, then the adversary chooses a bounded convex loss function $f_t:\+X \rightarrow [-1,1]$ and the learner suffers loss $f_t(\ARvec{x}_t)$. 

\paragraph{Regret and learning in games.} The classic goal is to minimize the (external) regret defined as
\begin{align*}
    \reg^T:=\max_{\ARvec{x} \in \+X} \left \{ \reg^T(\ARvec{x}) = \sum_{t=1}^T f_t(\ARvec{x}_t) - \sum_{t=1}^T f_t(\ARvec{x}) \right\}. \tag{standard regret}
\end{align*}
If the loss functions $\{f_t(\ARvec{x}) = \ARvec{c}_t ^\top \ARvec{x}\}$ are linear, the problem is known as the online linear optimization (OLO) problem. If additionally the action set $\+X = \Delta_d$ is the probability simplex, the problem reduces to the classic adversarial $d$-expert problem. The general OCO algorithm has regret $O(\sqrt{T})$, which is tight by the lower bound of $\Omega(\sqrt{T})$ even for the $2$-expert problem.

Algorithms for OCO/OLO have applications for learning in large-scale games and computing approximate equilibria. It is well-known that if players \emph{simultaneously} employ OLO algorithms to choose their strategies, then their time-averaged strategy converges to approximate Nash equilibria (NE) in two-player zero-sum games; their empirical distribution of play converges to approximate coarse correlated equilibria (CCE) in general-sum games~\citep{foster1997calibrated, freund1999adaptive, hart2000simple, cesa2006prediction}. The approximation depends on the average regret of the algorithm. Thus an algorithm with $\sqrt{T}$ regret in the adversarial setting implies $1/\sqrt{T}$ convergence rate in games. However, better convergence rate is possible in the game setting where every player employs the same algorithm. For two-player zero-sum games, the pioneering work of~\citep{daskalakis2011near} show $O(\log T/T)$ convergence rate to NE, while  optimistic online learning algorithms achieve the optimal $O(1/T)$ convergence rate~\citep{rakhlin2013predictable}, which can even be extended to last-iterate convergence~\citep{cai2023doubly}. For general-sum games, \citet{syrgkanis2015fast} show the first fast convergence rate of $O(1/T^{1/4})$, and a long line of work improve the convergence rate to $O(\log T/T)$~\citep{chen2020hedging, daskalakis2011near, anagnostides2022near-optimal, anagnostides2022uncoupled, farina2022near,cai2024near, soleymani2025faster, soleymani2025cautious}, some of which even hold for the stronger notion of correlated equilibrium and general convex games. A very recent work reports a $O(\sqrt{\log T}/T)$ convergence rate~\citep{tsuchiya2026sublogarithmic}. However, a major problem for learning in games remains open: can we achieve $O(1/T)$ convergence in general-sum games? This problem was open even for two-player games.  

\paragraph{Alternating regret and learning dynamics.} In this paper, we focus on \emph{alternation}, a simple trick for \emph{two-player} games where the players take turns to update their strategies (see \Cref{sec:alternating learning} for a formal definition). Compared to simultaneous learning,  alternating learning dynamics have strong empirical performance and are used for training superhuman AI models for poker~\citep{tammelin2014solving, bowling2015heads,brown2018superhuman, brown2019superhuman}. A theoretical explanation of the power of alternation, however, remained elusive until recent work. \citet{wibisono2022alternating} show that alternation with the classic Hedge algorithm~\citep{freund1997decision} gives a fast $O(1/T^{2/3})$ convergence rate in two-player zero-sum games. In their analysis, they propose a new regret notion capturing the nature of alternation, later termed as \emph{alternating regret} by~\citet{cevher2023alternation}. In the most general setting of OCO, the alternating regret is defined as follows:
\begin{align*}
    \reg^T_{\mathrm{Alt}}:=\max_{\ARvec{x} \in \+X} \left \{ \reg^T_{\mathrm{Alt}}(\ARvec{x}) = \sum_{t=1}^T (f_{t-1}(\ARvec{x}_t)+ f_t(\ARvec{x}_t)) - \sum_{t=1}^T (f_{t-1}(\ARvec{x}) + f_t(\ARvec{x})) \right\}. \tag{alternating regret}
\end{align*}
So the alternating regret can be seen as the standard regret with respect to the loss function $f_{t-1} + f_t$ in each iteration $t$ (define $f_0 = 0$). In alternating regret, the learner knows $f_{t-1}$ when they choose the strategy $\ARvec{x}_t$. This indicates that we may get $o(\sqrt{T})$ alternating regret even in the \emph{adversarial} setting where standard regret suffers $\Omega(\sqrt{T})$ bound. 

The work by~\citet{cevher2023alternation} show the first $o(\sqrt{T})$ alternating regret in the adversarial setting. They provide algorithms with $O(\log^{4/3} (dT) \cdot T^{1/3})$ regret for OLO over the probability simplex $\+X = \Delta_{d}$ (i.e., the $d$-expert problem), and even $O(\log T)$ when $\+X$ is the $\ell_2$ unit ball. Their results were improved by \citet{hait2025alternating}, who show that the classic Hedge algorithm already achieves an improved $O(\log^{2/3} d \cdot T^{1/3})$ alternating regret for the $d$-expert problem. Moreover, \citet{hait2025alternating} show that the continuous Hedge algorithm~\citep{narayanan2010random} achieves $\widetilde{O}(d^{2/3} T^{1/3})$ alternating regret for the general OCO problem over any convex set $\+X$. 

However, it is unknown whether $T^{1/3}$ is a fundamental limit for alternating regret. In this paper, we attack the main open question on alternating regret:
\begin{center}
     \emph{what is the minimax-optimal alternating regret for OLO/OCO?}
\end{center}
As remarked by~\citet{hait2025alternating},  ``\emph{This appears to be highly non-trivial even for very special cases---for example, even for $1$-dimensional OLO over $[-1,1]$ where $O(\log T)$ regret is achieved by \cite{cevher2023alternation}, figuring out whether the optimal bound is $\Theta(\log T)$ or $\Theta(1)$ appears to require new techniques.}"\footnote{One can easily show that $1$-dimensional OLO over $[-1,1]$ reduces to OLO over $\Delta_2$ (the 2-expert problem).}

\subsection{Our results}
In this paper, we settle the minimax-optimal alternating regret for OLO and OCO. We provide algorithms with optimal alternating regret with matching lower bounds, and fast alternating learning dynamics in two-player games.

\subsubsection{Optimal alternating regret}
\paragraph{Online linear optimization over the probability simplex.} We propose an online algorithm with $O(\log d)$ alternating regret for OLO over the probability simplex $\Delta_d$. The algorithm is anytime and does not need to know the time horizon $T$. It is remarkable that the alternating regret does not grow with $T$. 

\begin{theorem}[Constant alternating regret]\label{thm:main}
For any $T \ge 1$ and adaptive chosen linear loss sequence $\ARvec{c}_1,\ldots,\ARvec{c}_T \in[-1,1]^d$, the alternation-aware Hedge (AA-Hedge) algorithm (\Cref{alg:variance hedge}) over $\Delta_d$ guarantees that $
\AltReg^T\leq 8\log d$.
\end{theorem}
We also show that for any $T \ge d-1$, an adaptive adversary could force the alternating regret to be at least $\AltReg^T\ge 2 \log d -O(1)$ (\Cref{thm:lower}). Thus the minimax-optimal alternating regret for the $d$-expert problem is $\Theta(\log d)$. 

The AA-Hedge algorithm (\Cref{alg:variance hedge}) is a variant of the Hedge algorithm with non-trivial modifications. It has been shown that the vanilla Hedge algorithm can not achieve $o(T^{1/3})$ alternating regret~\citep{hait2025alternating}. In round $t$, \Cref{alg:variance hedge} fully exploits the fact that the known loss $\ARvec{c}_{t-1}$ is part of that round's loss. The played strategy $\ARvec{x}_t$ guarantees for any unknown $\ARvec{c}_t$, the instant loss $(\ARvec{c}_{t-1} + \ARvec{c}_t)^\top \ARvec{x}_t$ can be charged on a potential function without any additional terms. A potential analysis similar to that of the Hedge algorithm then concludes the constant alternating regret bound in general. 

\paragraph{Online convex optimization.} 
Motivated by the constant regret results for OLO over the simplex, we then study the general setting of online convex optimization over $d$-dimensional compact convex sets. We recall that the previous best is the $\widetilde{O}(d^{2/3}T^{1/3})$ regret by~\citet{hait2025alternating}, achieved by the continuous Hedge algorithm. In \Cref{sec:oco}, we extend \Cref{alg:variance hedge} and our analysis to the general online convex optimization setting with convex action set $\+X$, using ideas similar to the continuous Hedge algorithm. The resulting continuous AA-Hedge (\Cref{alg:main}) enjoys $O(d\log (1+T/d))$ alternating regret.
\begin{theorem}[Informal, alternating regret upper bound for OCO]
    For online convex optimization over any compact convex set $\+X \subseteq \-R^d$ with any bounded convex losses,  the continuous AA-Hedge (\Cref{alg:main}) has $O(d\log (1+T/d))$ alternating regret.
\end{theorem}
For the special case of OLO over a convex polytope with $N$ vertices, the continuous AA-Hedge algorithm enjoys a better alternating regret of $8 \log N$, which also recovers \Cref{thm:main}.

Somewhat surprisingly, we show that the $\Omega(\log T)$ dependence is necessary even when the set $\+X$ is the $2$-dimensional unit ball (\Cref{thm:oco-lower}). Using this lower bound, we further establish a matching lower bound for OCO over $d$-dimensional convex sets (\Cref{thm:oco-lower-dim}).

\begin{theorem}[Informal, alternating regret lower bound for OCO]
    There exists a compact convex set $\+X$ that lies in the $d$-dimensional unit $\ell_2$ ball such that for any OCO algorithm, there exists an oblivious adversary such that the algorithm suffers $\Omega(d\log(1+T/d))$ alternating regret
\end{theorem}

Our upper and lower bounds together establish the minimax-optimal $\Theta(d\log(1+T/d))$ alternating regret for OCO (\Cref{cor:oco-tight}).

\subsubsection{Fast alternating learning dynamics in games}\label{sec:alternating learning}
As an important application, our result implies in alternating learning dynamics with fast convergence in two-player games.
Consider a two-player \emph{convex game} with loss functions $u_1, u_2: \+X \times \+Y \rightarrow [-1,1]$ for the $x$-player and the $y$-player respectively. The loss $u_1(\ARvec{x},\ARvec{y})$ is convex in $\ARvec{x}$ for any $\ARvec{y} \in \+Y$ and the loss $u_2(\ARvec{x}, \ARvec{y})$ is convex in $\ARvec{y}$ for any $\ARvec{x} \in \+X$. The game is zero-sum if $u_1 = -u_2$ and otherwise is general-sum. An important subclass is \emph{normal-form games}.

\begin{example}[Normal-form games]
    In a two-player \emph{normal-form} game, $\+X$ and $\+Y$ are probability simplices and the loss functions are linear: $u_1(\ARvec{x}, \ARvec{y}) = \ARvec{x}^\top A \ARvec{y}, u_2(\ARvec{x}, \ARvec{y}) = \ARvec{x}^\top B\ARvec{y}$.  The game is \emph{zero-sum} if $A = -B$, otherwise it is \emph{general-sum}.
\end{example}
We recall the definition of approximate Nash equilibrium and coarse correlated equilibrium.
\begin{definition}[Nash equilibrium and coarse correlated equilibrium]
    A joint strategy profile $(\ARvec{x}, \ARvec{y})$ is an \emph{$\epsilon$-Nash equilibrium ($\epsilon$-NE)} if $u_1(\ARvec{x}, \ARvec{y}) \le \min_{\ARvec{x}' \in \+X} u_1(\ARvec{x}', \ARvec{y}) + \epsilon$ and $u_2(\ARvec{x}, \ARvec{y}) \le \min_{\ARvec{y}' \in \+X} u_2(\ARvec{x}, \ARvec{y}') + \epsilon$. A distribution $\sigma$ over joint strategy profiles $\+X \times \+Y$ is an \emph{$\epsilon$-coarse correlated equilibrium ($\epsilon$-CCE)} if $\-E_{(\ARvec{x},\ARvec{y})\sim \sigma}[u_1(\ARvec{x},\ARvec{y})] \le \min_{\ARvec{x}' \in \+X}\-E_{(\ARvec{x},\ARvec{y})\sim \sigma}[u_1(\ARvec{x}',\ARvec{y})] + \epsilon$ and $\-E_{(\ARvec{x},\ARvec{y})\sim \sigma}[u_2(\ARvec{x},\ARvec{y})] \le \min_{\ARvec{y}' \in \+Y}\-E_{(\ARvec{x},\ARvec{y})\sim \sigma}[u_2(\ARvec{x},\ARvec{y}')] + \epsilon$. 
\end{definition}

In a simultaneous learning dynamics, for every iteration $t \ge 1$, both players use an OLO/OCO algorithm to choose their actions $\ARvec{x}_t \in \+X$ and $\ARvec{y}_t  \in \+Y$ simultaneously and then they observe their loss functions $u_1(\cdot, \ARvec{y}_t)$ and $u_2(\ARvec{x}_t, \cdot)$. The \emph{alternating learning dynamics} is slightly different: in iteration $t$, the $x$-player chooses $\ARvec{x}_t$ first; then after seeing $\ARvec{x}_t$, the $y$-player chooses $\ARvec{y}_t$. More specifically, in each round $t \ge 1$,
\begin{itemize}
    \item[1.] the $x$-player chooses $\ARvec{x}_t$ based on $u_1(\cdot, \ARvec{y}_1), \ldots, u_1(\cdot, \ARvec{y}_{t-1})$;
    \item[2.] the $y$-player chooses $\ARvec{y}_t$ based on $u_2(\ARvec{x}_1, \cdot), \ldots, u_2(\ARvec{x}_{t-1}, \cdot), u_2(\ARvec{x}_t, \cdot)$;
\end{itemize}
Algorithms for OLO/OCO with $o(T)$ alternating regret imply convergence to approximate equilibria in two-player games, as observed by~\citep{wibisono2022alternating, cevher2023alternation} for NE in zero-sum games and by~\citep{hait2025alternating} for CCE in general-sum games. Formally, we have
\begin{theorem}[Theorem 1 and 2 in \citep{hait2025alternating}]\label{thm:alternating2Necce}
    Suppose that in the above alternating learning dynamics,  the $x$-player and the $y$-player use OCO algorithms with alternating regret bound $\AltReg^T(x)$ and $\AltReg^T(y)$, respectively. Then 
    \begin{itemize}
        \item[1.] for a zero-sum game,  the averaged strategy $(1/T\sum_{t=1}^T \ARvec{x}_t, 1/T\sum_{t=1}^T \ARvec{y}_t)$ is an $\epsilon$-NE with $\epsilon = O((\AltReg^T(x)+ \AltReg^T(y))/T)$;
        \item[2.] for a general-sum game, the uniform distribution over $\{(\ARvec{x}_t, \ARvec{y}_t), (\ARvec{x}_{t+1}, \ARvec{y}_t)\}_{t \in [T]}$ is an $\epsilon$-CCE with $\epsilon = O(\max\{\AltReg^T(x), \AltReg^T(y)\}/T)$
    \end{itemize}
\end{theorem}

As a corollary, our results imply alternating learning dynamics with fast convergence in games. Notably, we give the first learning dynamics with $O(1/T)$ convergence to CCE in two-player general-sum normal-form games, while existing results suffer additional $\mathrm{poly}\log T$ factors. 
\begin{corollary}[Fast convergence in two-player games]
    For two-player normal-form games where each player has at most $d$ actions, there is an alternating learning dynamics with $O(\log d/T)$ convergence to NE in zero-sum games and $O(\log d/T)$ convergence to CCE in general-sum games.
\end{corollary}

\subsection{Further related works} 
Alternation has been studied in two-player zero-sum games and the related min-max optimization problem, mostly in the unconstrained setting where the strategy set is the whole Euclidean space. There is a line of work on the convergence of alternating gradient-descent-ascent in the unconstrained setting \citep{bailey2020finite, zhang2022near, lee2024fundamental, feng2025continuoustime, shugart2025negative}, the simplex-constrained setting~\citep{nan2026on}, and manifold setting~\citep{xu2026riemannian}. The difference is that our work considers alternating regret in the more challenging setting of adversarial online learning, which implies convergence in the game setting.

\section{Constant alternating regret for OLO over the simplex}\label{sec:simplex}
In this section, we focus on the important case of online linear optimization over the simplex, i.e., the adversarial expert problem and present an algorithm with constant alternating regret.

We first introduce some notations. The action set is the $d$-dimensional probability simplex $
\Delta_d=\left\{\ARvec{x}\in\R^d:x_i\geq 0\ \text{and}\ \sum_{i=1}^d x_i=1\right\}$. We define $\ARvec{c}_0 = \ARzero$ to be the all-0 vector. In each iteration $t \ge 1$, the learner chooses $\ARvec{x}_t \in \Delta_d$ based on the history $\{\ARvec{c}_1, \ldots, \ARvec{c}_{t-1}\}$. After observing $\ARvec{x}_t$, the adversary \emph{adaptively} chooses $\ARvec{c}_t \in [-1,1]^d$. The learner then observes $\ARvec{c}_t$ and suffers loss $(\ARvec{c}_{t-1} + \ARvec{c}_t)^\top \ARvec{x}_t$. In this case, the alternating regret can be written as 
\begin{equation}\label{eq:regret}
\AltReg^T
=
\sum_{t=1}^T (\ARvec{c}_{t-1}+\ARvec{c}_t)^\top \ARvec{x}_t
-
\min_{\ARvec{x}\in\Delta_d}\sum_{t=1}^T(\ARvec{c}_{t-1}+\ARvec{c}_t)^\top \ARvec{x}.
\end{equation}
Because the objective is linear, the comparator in \cref{eq:regret} is one of the $d$ pure actions.

\subsection{The alternation-aware Hedge algorithm}
We propose the alternation-aware Hedge algorithm, which introduce a bias toward the most recent known loss vector $\ARvec{c}_{t-1}$ when choosing $\ARvec{x_t}$. The algorithm is in \Cref{alg:variance hedge}, where we define $\ARvec{L}_0 = \ARzero$ and $\ARvec{L}_{t} = \sum_{s=1}^{t} (\ARvec{c}_{s-1} + \ARvec{c}_s)$ as the effective cumulative loss for $t \ge 1$. We first show that each step in \Cref{alg:variance hedge} is well-defined and feasible: (1) $\theta_t$ is unique; (2) the played strategy $\ARvec{x}_t$ lies in the simplex $\Delta_d$. 
\begin{algorithm}[H]
\caption{alternation-aware Hedge}
\label{alg:variance hedge}
\begin{algorithmic}[1]
\State Set $\ARvec{L}_0 = \ARvec{0}$, $\eta=1/8$, and $\beta=1/32$.
\For{$t=1,2,\ldots$}
  \State Compute $\theta_t = \argmax_{\theta\in\R} \log
\sum_{i=1}^d
\exp\!\left[-\eta L_{t-1,i}-\beta(c_{t-1,i}-\theta)^2\right]$ 
  \State Set $\ARvec{p}_t$ such that $p_{t, i} \propto \exp\!\left[-\eta L_{t-1,i}-\beta(c_{t-1,i}-\theta_t)^2\right] $
  \State Play $\ARvec{x}_t$ such that $x_{t,i}
=p_{t,i}\left(1-\frac{c_{t-1,i}-\theta_t}{2}\right)$.
  \State Observe $\ARvec{c}_t$ and update $\ARvec{L}_t \gets \ARvec{L}_{t-1}+\ARvec{c}_{t-1}+\ARvec{c}_t$.
\EndFor
\end{algorithmic}
\end{algorithm}

\begin{lemma}\label{lem:center-feasible}
Assume $\beta < 1/2$. In \Cref{alg:variance hedge}, $\theta_t$ is unique and $\theta_t = \sum_{i=1}^d p_{t,i} c_{t-1,i}$. Moreover, $\ARvec{x}_t \in \Delta_d$.
\end{lemma}

\begin{proof}
Define $F(\theta) = \log \sum_{i=1}^d
\exp\!\left[-\eta L_{t-1,i}-\beta(c_{t-1,i}-\theta)^2\right]$ and $\ARvec{p}(\theta) \in \Delta_d$ such that $p_i(\theta) \propto  \exp\!\left[-\eta L_{t-1,i}-\beta(c_{t-1,i}-\theta)^2\right]$. We can calculate both the first-order and the second-order derivative of $F(\theta)$:
\begin{align*}
    F'(\theta) &= 2\beta(\-E_{\ARvec{p}(\theta)}(\ARvec{c}_{t-1}) - \theta) \\
    F''(\theta) &= 4\beta^2 \cdot\operatorname{Var}_{\ARvec{p}(\theta)}(\ARvec{c}_{t-1}) - 2\beta
\end{align*}
where we define the expectation $\-E_{\ARvec{p}(\theta)}(\ARvec{c_{t-1}}) = \sum_i p_i(\theta) c_{t-1,i}$ and $\operatorname{Var}_{\ARvec{p}(\theta)}(\ARvec{c_{t-1}}) = \sum_i p_i(\theta) (c_{t-1,i} - \-E_{\ARvec{p}(\theta)}(\ARvec{c}_{t-1}))^2$. Since $c_{t-1, i} \in [-1,1]$, the variance term $\operatorname{Var}_{\ARvec{p}(\theta)}(\ARvec{c_{t-1}}) \le 1$. Since $\beta < \frac{1}{2}$, we have $F''(\theta) < 0$ and $F$ is thus strictly concave. Moreover, $F(\theta) \rightarrow -\infty$ as $|\theta| \rightarrow +\infty$. Thus $F$ has a unique maximizer $\theta_t$ and $\ARvec{p}_t = \ARvec{p}(\theta_t)$. The first-order optimality condition $F'(\theta_t) = 0$ implies $\theta_t = \sum_{i=1}^d p_{t,i} c_{t-1,i} \in [\min_i c_{t-1,i}, \max_{i} c_{t-1, i}]$.

Both $c_{t-1,i}$ and $\theta_t$ lie in $[-1,1]$, so $x_{t,i} =p_{t,i} (1-(c_{t-1,i}-\theta_t)/2)\geq0$ for all $i \in [d]$. Moreover, $\theta_t = \sum_{i=1}^d p_{t,i} c_{t-1,i} $ implies $\sum_{i=1}^d x_{t,i}= 1$. Hence $\ARvec{x}_t\in\Delta_d$.
\end{proof}

\subsection{A potential proof of $O(\log d)$ alternating regret}
We use a potential function argument to bound the alternating regret. Define the alternating regret vector $\ARvec{R}_{t}$ such that $R_{t,i} = \sum_{s=1}^{t} (\ARvec{c}_{s-1}+\ARvec{c}_s)\top \ARvec{x}_s - \sum_{s=1}^t(c_{s-1, i} + c_{s,i})$ for all $i \in [d]$. 

\paragraph{Potential function.} The potential function is
\begin{equation}\label{eq:potential}
\Phi_t:=\max_{\theta\in\R}\log \left(
\sum_{i=1}^d
\exp\!\left[
\eta R_{t-1,i}-\beta(c_{t-1,i}-\theta)^2
\right]\right) = \log \left(\sum_{i=1}^d
\exp\!\left[
\eta R_{t-1,i}-\beta(c_{t-1,i}-\theta_t)^2
\right]\right).
\end{equation}
The equality holds by definition of $\theta_t$ in \Cref{alg:variance hedge} and the fact that constant terms do not affect the maximizer. 
Initially, $\Phi_1=\log d$.

\begin{theorem}[Monotonicity of potential]\label{thm:potential}
Assume $\eta \le \frac{1}{8}$ and $\beta = \eta/4$. For every round $t \ge 1$, we have $\Phi_{t+1}\leq\Phi_t \le \log d$.
\end{theorem}

\begin{proof}
Let us define $\ARvec{\xi}_t=\ARvec{c}_{t-1}-\theta_t\one$ and $\ARvec{\xi}_{t+1}=\ARvec{c}_t-\theta_{t+1}\one$.
Applying \cref{lem:center-feasible} at rounds $t$ and $t+1$ gives $\ARvec{\xi}_t,\ARvec{\xi}_{t+1}\in[-2,2]^d$, $\ARvec{p}_t^\top\ARvec{\xi}_t=0$, and the update gives $x_{t,i}=p_{t,i}(1-\xi_{t,i}/2)$. We have the following identity:
\[
R_{t,i} - R_{t-1,i} = (\ARvec{c}_{t-1}+\ARvec{c}_t)^\top \ARvec{x}_t-c_{t-1,i}-c_{t,i}
=(\ARvec{\xi}_t+\ARvec{\xi}_{t+1})^\top \ARvec{x}_t
-\xi_{t,i}-\xi_{t+1,i}.
\]
By definition of $\Phi_t$ and $p_{t,i} \propto \exp(\eta R_{t-1,i} - \beta(c_{t-1,i}-\theta_t)^2)$, we have
\begin{align*}
    \Phi_{t+1} - \Phi_t &= \log\left(\frac{\sum_{i=1}^d \exp \left(\eta R_{t,i} - \beta(c_{t,i} -\theta_{t+1})^2 \right)}{\sum_{i=1}^d \exp \left(\eta R_{t-1,i} - \beta(c_{t-1,i} -\theta_{t})^2 \right)} \right) \\
    &=  \log \left( \sum_{i=1}^d p_{t,i}
\exp\!\left(
\eta (R_{t,i} - R_{t-1,i})
+\beta(\xi_{t,i}^2-\xi_{t+1,i}^2)
\right) \right).  \\
     &= 
\log \left( \sum_{i=1}^d p_{t,i}
\exp\!\left(
\eta((\ARvec{\xi}_t+\ARvec{\xi}_{t+1})^\top \ARvec{x}_t
-\xi_{t,i}-\xi_{t+1,i})
+\beta(\xi_{t,i}^2-\xi_{t+1,i}^2)
\right) \right).
\end{align*}
We then use the following lemma to show that $\Phi_{t+1} \le \Phi_t$.

\begin{lemma}\label[lemma]{lem:one-step}
Assume $\eta \le \frac{1}{8}$ and $\beta = \eta/4$. Then for any $t \ge 1$, we have
\begin{equation*}
\log \left( \sum_{i=1}^d p_{t,i}
\exp\!\left(
\eta\bigl((\ARvec{\xi}_t+\ARvec{\xi}_{t+1})^\top \ARvec{x}_t
-\xi_{t,i}-\xi_{t+1,i}\bigr)
+\beta(\xi_{t,i}^2-\xi_{t+1,i}^2)
\right)  \right) 
\leq 0.
\end{equation*}
\end{lemma}

\begin{proof}
Let us define the function $H_t: [-2, 2]^d \rightarrow \-R$ as follows:
\[
H_t(\ARvec{z})
=
\log\sum_{i=1}^d p_{t,i}
\exp\!\left(
\eta\bigl((\ARvec{\xi}_t+\ARvec{z})^\top \ARvec{x}_t
-\xi_{t,i}-z_i\bigr)
+\beta(\xi_{t,i}^2-z_i^2)
\right).
\]
It suffices to show $H_t(\ARvec{\xi}_{t+1}) \le0$.

Define $\ARvec{q}(\ARvec{z})$ be the distribution obtained by normalizing the summands:
\[
q_j(\ARvec{z}) = \frac{p_{t,j}
\exp\!\left(
\eta\bigl((\ARvec{\xi}_t+\ARvec{z})^\top \ARvec{x}_t
-\xi_{t,j}-z_j\bigr)
+\beta(\xi_{t,j}^2-z_j^2) \right)}{\sum_{i=1}^d p_{t,i}
\exp\!\left(
\eta\bigl((\ARvec{\xi}_t+\ARvec{z})^\top \ARvec{x}_t
-\xi_{t,i}-z_i\bigr)
+\beta(\xi_{t,i}^2-z_i^2) \right)}.
\]
We can calculate the gradient $\nabla H_t(\ARvec{z})=\eta \ARvec{x}_t+\ARvec{q}(\ARvec{z})\odot(-\eta\one-2\beta \ARvec{z})$, where we use the point-wise product notation $(\ARvec{a} \odot \ARvec{b})_i = a_i b_i$. We then calculate the Hessian
\begin{align*}
\nabla^2H_t(\ARvec{z})
&=\operatorname{Diag}\!\left[
q_i(\ARvec{z})\left((-\eta-2\beta z_i)^2-2\beta\right)
\right]\notag\\
&\quad-
\bigl(\ARvec{q}(\ARvec{z})\odot(-\eta\one-2\beta \ARvec{z})\bigr)
\bigl(\ARvec{q}(\ARvec{z})\odot(-\eta\one-2\beta \ARvec{z})\bigr)^\top,
\end{align*}
where we use notation $\operatorname{Diag}[a_i]$ to denote the square $d\times d$ diagonal matrix whose $i$-th diagonal is $a_i$ for $i \in [d]$. For any $\ARvec{z}\in [-2,2]^d$, since $2\beta \ge \max\{ (\eta- 4\beta)^2, (\eta +4\beta)^2\}$, we have $\nabla^2 H_t(\ARvec{z})$ is negative semidefinite. Therefore $H_t$ is concave over $[-2,2]^d$.

Let us analyze $\ARvec{z} = -\ARvec{\xi}_t$. By definition, we have $H_t(-\ARvec{\xi}_t) = 0$ and $\ARvec{q}(-\ARvec{\xi_t}) = \ARvec{p}_t$. Then the gradient $\nabla H_t(-\ARvec{\xi}_t)$ satisfies
\begin{align*}
\nabla H_t(-\ARvec{\xi}_t) & = \eta \ARvec{x_t} + \ARvec{p}_t \odot (-\eta + 2\beta \ARvec{\xi}_t) \\
&= \eta \ARvec{p_t} \odot (1- \ARvec{\xi}_t/2) + \ARvec{p}_t \odot (-\eta + 2\beta \ARvec{\xi}_t) \\ 
&= (2\beta-\eta/2) \ARvec{p_t} \odot \ARvec{\xi}_t  \\
& = \ARvec{0}.
\end{align*}
Thus $-\ARvec{\xi}_t$ is the maximizer of $H_t$ over $[-2,2]^d$ since $H_t$ is concave. Then we have $H_t(\ARvec{\xi_{t+1}}) \le H_t(-\ARvec{\xi}_t) = 0$. This completes the proof.
\end{proof}
By \Cref{lem:one-step}, we conclude that $\Phi_{t+1}\le \Phi_t$ for every $t \ge 1$.
\end{proof}

\begin{proof}[Proof of \Cref{thm:main}]
The choices $\eta = 1/8$ and $\beta = 1/32$ in \Cref{alg:variance hedge} satisfy the conditions in \Cref{lem:center-feasible} and \Cref{thm:potential}. Fix any $T \ge 1$. By \Cref{thm:potential}, we have $\Phi_{T+1}\leq\Phi_1=\log d$. Fix any pure action $i \in [d]$. By setting $\theta = c_{T,i}$ in \cref{eq:potential},  we have $\Phi_{T+1} \ge \eta R_{T,i}$.  
Therefore $R_{T,i}\leq\log d/\eta=8\log d$ for every $i$. Thus the alternating regret is at most $8 \log d$ for any $T \ge 1$.
\end{proof}

\subsection{An $\Omega(\log d)$ lower bound}
We construct an adaptive adversary such that any algorithm suffers $\Omega(\log d)$ alternating regret. Together with \Cref{thm:main}, we show that the minimax-optimal alternating regret for OLO over the simplex is $\Theta(\log d)$. In this lower bound construction, the adversary adaptively chooses the loss function so that the alternating regret in each round $t \in [d-1]$ is at least $2/(d+1-t)$. This then implies the $\Omega(\log d)$ bound.  

\begin{theorem}[Logarithmic dependence is necessary]\label{thm:lower}
For every $d\geq2$, every horizon $T\geq d-1$, and every online algorithm, an adaptive adversary can choose linear costs in $[-1,1]^d$ such that
\[
\AltReg^T\geq2\left(\sum_{m=1}^d\frac1m-1\right )= 2\log d -O(1).
\]
\end{theorem}

\begin{proof}
Define $S_0=[d]$.  
In round $1\le t\le d-1$, after observing $\ARvec{x}_t$, the adversary computes $j_t\in\argmax_{j\in S_{t-1}}x_{t,j}$,  update $S_t:=S_{t-1}\setminus\{j_t\}$, and sets the loss $\ARvec{c}_t$
\begin{equation*}
c_{t,i}
=
\begin{cases}
-1, & i\in S_t,\\
+1, & i\notin S_t.
\end{cases}
\end{equation*}
The set $S_{d-1}$ contains one action, denoted by $i^\star$. For $t\ge d$, the adversary always chooses $\ARvec{c}_t = \ARvec{c}_{d-1}$.

At $t=1$, we have $\ARvec{c}_1^\top \ARvec{x}_1-c_{1,i^\star}=2x_{1,j_1}\geq2/d$.  For $2\le  t\le d-1$, the effective cost $\ARvec{g}_t:=\ARvec{c}_{t-1}+\ARvec{c}_t$ is
\begin{equation*}
g_{t,i}
=
\begin{cases}
-2, & i\in S_t,\\
0, & i=j_t,\\
+2, & i\notin S_{t-1}.
\end{cases}
\end{equation*}
Thus the learner's one-round alternating regret relative to $i^\star$ is 
\begin{align*}
    \ARvec{g}_t^\top \ARvec{x}_t-g_{t,i^\star}
=2x_{t,j_t}+4\sum_{i\notin S_{t-1}}x_{t,i} \ge 2(1-\sum_{i\notin S_{t-1}}x_{t.i})/(d+1-t) + 4\sum_{i\notin S_{t-1}}x_{t,i}  \ge \frac{2}{d+1-t},
\end{align*}
where the first inequality is because $x_{t,j_t} = \max_{j \in S_{t-1}}x_{t,j} \ge (1-\sum_{i\notin S_{t-1}}x_{t,i})/|S_{t-1}| = (1-\sum_{i\notin S_{t-1}}x_{t,i})/(d+1-t)$. Therefore, we have
\begin{align*}
    \AltReg^{d-1}\ge \sum_{t=1}^{d-1} (\ARvec{g}_t^\top \ARvec{x}_t-g_{t,i^\star}) \ge \sum_{t=1}^{d-1} \frac{2}{d+1-t} = 2\left(\sum_{m=1}^d\frac1m-1\right )= 2\log d -O(1).
\end{align*}
Moreover, since $\ARvec{c}_t = \ARvec{c}_{d-1}$ for all $t \ge d$, we have $\AltReg^{t} \ge \AltReg^{d-1} \ge 2\log d - O(1)$. 
\end{proof}

\section{Alternating regret for online convex optimization}\label{sec:oco}
In this section, we extend our result to the more general setting of online convex optimization.

\paragraph{Online convex optimization.} We first introduce some notations. Let $\X\subseteq\R^d$ be a nonempty compact convex set. We denote by $\Delta_{\+X}$ the set of Borel probability measures over $\+X$. For $\rho,\mu\in\Delta_{\+X}$, we say that $\mu$ \emph{dominates} $\rho$ if $\rho$ is absolutely continuous with respect to $\mu$, i.e., $\rho(A)=0$ for every measurable set $A$ such that $\mu(A)=0$. We write $\rho\ll\mu$ if $\mu$ dominates $\rho$.

At round $t\ge1$, the learner chooses $\ARvec{x}_t\in\+X$ using $f_1,\ldots,f_{t-1}$, and then the adversary adaptively chooses a continuous convex loss function $f_t:\+X\to[-1,1]$. The alternating regret to $\ARvec{u}\in\+X$ is
\begin{equation}\label{eq:alt-regret}
  \AltReg^T(\ARvec{u})
  :=\sum_{t=1}^{T}\bigl(f_{t-1}(\ARvec{x}_t)+f_t(\ARvec{x}_t)\bigr)
  -\sum_{t=1}^{T}\bigl(f_{t-1}(\ARvec{u})+f_t(\ARvec{u})\bigr)
  =\sum_{t=1}^{T}\bigl(g_t(\ARvec{x}_t)-g_t(\ARvec{u})\bigr),
\end{equation}
where we define $f_0\equiv0$ and the effective loss function $g_t:=f_{t-1}+f_t$ for $t\ge1$. For a distribution $\rho\in\Delta_{\+X}$, we write $\AltReg^T(\rho):=\E_{\ARvec{u}\sim\rho}[\AltReg^T(\ARvec{u})]$. We also use the shorthand $\-E_{\mu}[f]:= \-E_{x\sim \mu}[f(x)]$.

\subsection{The continuous  alternation-aware Hedge algorithm}
We propose the continuous  alternation-aware Hedge algorithm (\Cref{alg:main}), a continuous generalization of \Cref{alg:variance hedge}. Fix a reference distribution $\mu\in\Delta_{\+X}$, and define $L_0\equiv0$ and $L_t(\ARvec{u}):=\sum_{s=1}^t g_s(\ARvec{u})$. The algorithm replaces the finite weight vector $\ARvec{p}_t$ by a probability measure $P_t$, corrects it to a probability measure $Q_t$, and plays the barycenter of $Q_t$. When $\+X=\Delta_d$, $\mu$ is uniform over the pure actions, and the losses are linear, the algorithm reduces to \Cref{alg:variance hedge}.

As in \Cref{alg:variance hedge}, we first show that the algorithm is well defined and feasible. The proof is deferred to \cref{sec:oco-proofs}.
\begin{lemma}\label{lem:continuous-center}
Assume $\beta<1/2$. In \Cref{alg:main}, $\theta_t$ is unique and satisfies $\theta_t=\E_{P_t}[f_{t-1}]$. Moreover, $Q_t$ is a probability measure and $\ARvec{x}_t\in\+X$.
\end{lemma}

\begin{algorithm}[H]
\caption{Continuous alternation-aware Hedge (continuous AA-Hedge)}
\label{alg:main}
\begin{algorithmic}[1]
\State Given action set $\+X$ and the reference distribution $\mu \in \Delta_{\+X}$
\State Set $f_0\equiv0$, $L_0\equiv0$, $\eta=1/8$, and $\beta=1/32$.
\For{$t=1,2,\ldots$}
  \State Compute $\theta_t$ as the unique maximizer of
  \[
  \log\int_{\+X}\exp\!\left[-\eta L_{t-1}(\ARvec{u})-\beta(f_{t-1}(\ARvec{u})-\theta)^2\right]\mu(\dd\ARvec{u}).
  \]
  \State Set $\xi_t(\ARvec{u})=f_{t-1}(\ARvec{u})-\theta_t$ and form the probability measure
  \[
  P_t(\dd\ARvec{u})\propto
  \exp\!\left[-\eta L_{t-1}(\ARvec{u})-\beta\xi_t(\ARvec{u})^2\right]\mu(\dd\ARvec{u}).
  \]
  \State Set $Q_t(\dd\ARvec{u})=\left(1-\xi_t(\ARvec{u})/2\right)P_t(\dd\ARvec{u})$.
  \State Play the barycenter $\ARvec{x}_t=\int_{\X}\ARvec{u}\,Q_t(\dd \ARvec{u})$.
  \State Observe $f_t$ and update $L_t\gets L_{t-1}+f_{t-1}+f_t$.
\EndFor
\end{algorithmic}
\end{algorithm}

\subsection{Alternating regret bounds}
We first prove an upper bound on the distributional alternating regret $\AltReg^T(\rho)$ for any $\rho \ll \mu$. The alternating regret is at most $O(\KL(\rho \| \mu))$ where the KL divergence over measures is defined as $\KL(\rho \| \mu):= \int_{\+X} \log \frac{\dd \rho}{\dd \mu} \dd \rho$. This is the continuous analogue of comparing with a pure action in \Cref{thm:main}. The proof is similar to \Cref{thm:main} and is deferred to \cref{sec:oco-proofs}. 

\begin{theorem}[Distributional alternating-regret bound]\label{thm:distributional}
For every horizon $T\ge1$ and every $\rho\ll\mu$, \Cref{alg:main} guarantees
\begin{equation}\label{eq:distributional-bound}
  \AltReg^T(\rho) \le  8\,\KL(\rho\|\mu) + \frac{\operatorname{Var}_{\rho}(f_T)}{4} \le 8\,\KL(\rho\|\mu) + \frac{1}{4} .
\end{equation}
\end{theorem}

Then we instantiate the results with $\mu$ being the uniform probability measure over $\+X$\footnote{This requires $\+X \subseteq \-R^d$ to be full-dimensional. If $\+X$ is not full-dimensional, let $\mu$ be the uniform probability measure over its affine hull and define $d$ as the dimension of its affine.},  to get alternating regret bound $\max_{\ARvec{u}\in \+X} \AltReg^T(\ARvec{u}) = O(d \log T)$.

\begin{theorem}[$O(d\log (1+T/d))$ alternating regret for OCO]\label{cor:pointwise}
For every horizon $T \ge 1$, \Cref{alg:main} with the uniform probability measure $\mu \in \Delta_{\+X}$ guarantees
\begin{equation}\label{eq:pointwise-bound}
  \max_{\ARvec{u} \in \+X}\AltReg^T(\ARvec{u})
  \le
  8d\left(1+\logp\frac{T}{2d}\right) + \frac{1}{4},
\end{equation}
where $\logp z:=\max\{0,\log z\}$. Since $\max_{\ARvec{u} \in \+X}\AltReg^T(\ARvec{u}) \le 4T$ trivially, we have
\begin{align*}
    \max_{\ARvec{u} \in \+X}\AltReg^T(\ARvec{u}) = O\left( \min\left\{T, d\left (1+\log_{+} \frac{T}{2d}\right) \right\} \right) = O(d \log(1+T/d))
\end{align*}
\end{theorem}

\begin{proof}
Fix $\ARvec{u} \in \+X$ and $\delta \in (0,1)$. Let $\rho_{\ARvec{u},\delta}$ be the uniform probability measure on the set $(1-\delta)\ARvec{u}+\delta\X$. The affine map $\ARvec{z}\mapsto(1-\delta)\ARvec{u}+\delta\ARvec{z}$ scales $d$-dimensional volume by $\delta^d$, so $\KL(\rho_{\ARvec{u},\delta}\|\mu)=d\log(1/\delta)$. For any $\ARvec{v}=(1-\delta)\ARvec{u}+\delta\ARvec{z}$ with $\ARvec{z} \in \+X$, the convexity of $f_t\in[-1,1]$ gives
\[
f_t(\ARvec{v})\le(1-\delta)f_t(\ARvec{u})+\delta f_t(\ARvec{z})\le f_t(\ARvec{u})+2\delta.
\]
Thus, we have $\E_{\rho_{\ARvec{u},\delta}}[f_t]-f_t(\ARvec{u}) \le2\delta$ for any $t \ge 1$. 

Combining the above with \Cref{thm:distributional}, we have
\begin{align*}
    \AltReg^T(\ARvec{u}) &=\AltReg^T(\rho_{\ARvec{u}, \delta}) +2\sum_{t=1}^{T-1}\left(\E_{\rho_{\ARvec{u}, \delta}}[f_t]-f_t(\ARvec{u})\right) +\left(\E_{\rho_{\ARvec{u}, \delta}}[f_T]-f_T(\ARvec{u})\right).\\
    & \le 8\KL(\rho_{\ARvec{u}, \delta}\| \mu) + \frac{1}{4} + 4T\delta \\
    &\le 8d\log(1/\delta) + \frac{1}{4} + 4T \delta.
\end{align*}
Choosing $\delta=\min\{1,2d/T\}$ proves the claim.
\end{proof}

Moreover, for OLO over a polytope $\+X$ with $N$ vertices,  \Cref{alg:main} with $\mu$ being the uniform distribution over these vertices has alternating regret $O(\log N)$. 

\begin{corollary}[Constant alternating regret for polyhedral OLO]\label[corollary]{cor:polytope}
Let $\X=\operatorname{conv}\{\ARvec{v}_1,\ldots,\ARvec{v}_N\}$ be a  polytope. Consider  linear losses $f_t(\ARvec{x})=\langle \ARvec{c}_t,\ARvec{x}\rangle$, with $f_t(\ARvec{x})\in[-1,1]$ for every $\ARvec{x}\in\+X$. Running \Cref{alg:main} with $\mu$ being the uniform distribution on $\{\ARvec{v}_1,\ldots,\ARvec{v}_N\}$ gives
\[
  \max_{\ARvec{u}\in\X}\AltReg^T(\ARvec{u})\le8\log N.
\]
In particular, for $\X=\Delta_d$ and $\ARvec{c}_t\in[-1,1]^d$, \Cref{alg:main} coincides with \Cref{alg:variance hedge} and recovers the bound $8\log d$.
\end{corollary}

\begin{proof}
The cumulative comparator loss is linear, so it attains its minimum over the polytope at one of its vertices. Applying \Cref{thm:distributional} with $\rho = \delta_{\ARvec{v}_j}$ and note that $\KL(\delta_{\ARvec{v}_j}\|\mu)=\log N$ and $\operatorname{Var}_{\delta_{\ARvec{v}_j}}(f_T)=0$. Thus we have $\max_{\ARvec{u}\in\X}\AltReg^T(\ARvec{u}) = \max_{j \in [N]} \AltReg^T(\ARvec{v}_j) \le 8\log N$.
\end{proof}

\section{Matching lower bounds for OCO}
In this section, we prove a alternating regret lower bound of $\Omega(d\log(1+T/d))$ for OCO with bounded convex losses over a $d$-dimensional convex set $\+X \subseteq B_d(1):=\{\ARvec{x} \in \-R^d: \|\ARvec{x}\|_2 \le 1\}$. In \Cref{sec:oco-lower},  we first prove a $\Omega(\log T)$ lower bound when the action set is the $2$-dimensional unit ball $B_2(1)$ (\Cref{thm:oco-lower}). Then, in \Cref{sec:oco-lower-dim}, we extend the $2$-dimensional construction to $d$-dimensional convex sets and prove the $\Omega(d\log(1+T/d))$ lower bound.

\subsection{An $\Omega(\log T)$ lower bound for OCO over $2$-dimensional ball}\label{sec:oco-lower}

In  this section, we show that the $\log T$ dependence in \cref{cor:pointwise} is unavoidable. The lower bound already holds on the two-dimensional Euclidean unit ball $
  B_2(1):=\{\ARvec{x}\in\R^2:\lVert\ARvec{x}\rVert_2\leq1\}$.
Throughout this section, let $\AltReg^T:=\max_{\ARvec{u}\in B_2(1)}\AltReg^T(\ARvec{u})$.

\begin{theorem}[Logarithmic lower bound for OCO]\label{thm:oco-lower}
For every integer $T\geq1$ and every possibly randomized online algorithm $\mathcal A$, there is an sequence of continuous convex losses $f_1,\ldots,f_T:B_2(1)\to[0,1]$ such that
\[
  \E_{\mathcal A}[\AltReg^T]
  \geq\frac{1}{125}\left\lfloor\log_{16}T\right\rfloor.
\]
The expectation is only over the randomness of $\mathcal A$.
\end{theorem}

\paragraph{Proof overview.} We in fact prove a stronger lower bound:
\begin{align*}
    \-E_{\+A} [\AltReg^T] \ge \-E_{\+A} \left[\sum_{t=1}^T f_t(\ARvec{x}_t) -2 \min_{\ARvec{u} \in B_2(1)} \sum_{t=1}^T f_t(\ARvec{u})\right] \ge \Omega(\log T).
\end{align*}
The first inequality holds since the losses are non-negative. The factor $2$ in the cumulative loss of the best fixed action is thus not essential. Our lower bound can be easily extended to the regret notion where we compete with $c \cdot \min_{\ARvec{u} \in B_2(1)} \sum_{t=1}^T f_t(\ARvec{u})$ for any constant $c > 0$.

The construction of the loss sequence proceeds in $K = O(\log T)$ blocks. Each block $k \in [K]$ has a center $\ARvec{u}_k$ on the boundary of $B_2(1)$ and an angular scale $\alpha_k$. Within block $k$, each round randomizes between a \emph{common loss}, which is minimized at $\ARvec{u}_k$, and a \emph{rare loss}, which is largest at $\ARvec{u}_k$ but vanishes after a rotation by $\alpha_k$. Fixing a small constant $\nu$, we choose the rare-loss probability $O(\alpha_k^2)$ and the block length proportional to $O(\nu/\alpha_k^2)$. This choice ensures that any algorithm incurs expected loss at least $\nu$ in every block (\Cref{lem:oco-lb-gadget}) and thus cumulative loss of at least $\Omega(\nu K)$ (\Cref{lem:oco-lb-algorithm-cost}). 

After a block, we rotate the center by $\alpha_k$ if the block contains a rare loss, and otherwise leave it unchanged. We maintain geometrically decreasing $\alpha_k = \sqrt{\nu}/4^{k-1}$. We use the endpoint $\ARvec{u}_{K+1}$ after all $K$ rotations in these blocks as the hindsight comparator. We prove that $\ARvec{u}_{K+1}$ has $0$ loss on every realized rare, while incurring only $O(\nu^2)$ expected common loss per block (\Cref{lem:oco-lb-comparator-cost}). Thus the (expected) alternating regret of any algorithm is at least $\Omega(\nu K -  \nu^2 K) = \Omega(K) = \Omega(\log T)$.


\paragraph{One-round loss pair.}
Fix a unit vector $\ARvec{u}\in\partial B_2(1)$ and an angle $0<\alpha\leq1$. Define
\begin{align*}
  \ell_{\ARvec{u}}(\ARvec{x})
  &:=\frac{1-\langle\ARvec{u},\ARvec{x}\rangle}{2},\\
  h_{\ARvec{u},\alpha}(\ARvec{x})
  &:=\max\left\{0,\frac{\langle\ARvec{u},\ARvec{x}\rangle-\cos\alpha}{1-\cos\alpha}\right\}, \\
  p_\alpha&:=\frac{1-\cos\alpha}{3-\cos\alpha}.
\end{align*}
We refer to $\ell_{\ARvec{u}}$ as the common loss and $h_{\ARvec{u},\alpha}$ as the rare loss. We will randomly choose the common loss $\ell_{\ARvec{u}}$ with probability $1 - p_\alpha$ and $h_{\ARvec{u},\alpha}$ with probability $p_\alpha$. We first present some useful properties of these loss functions.

\begin{lemma}[Properties of one-round loss pair]\label{lem:oco-lb-gadget}
The common loss $\ell_{\ARvec{u}}$ and the rare loss $ h_{\ARvec{u}, \alpha}$ are continuous and convex and take values in $[0,1]$. Moreover, $h_{\ARvec{u},\alpha}(\ARvec{u})=1$, and $h_{\ARvec{u},\alpha}(\ARvec{v})=0$ whenever $\langle\ARvec{u},\ARvec{v}\rangle\leq\cos\alpha$. For every $\ARvec{x}\in B_2(1)$, we have
\[
  (1-p_\alpha)\ell_{\ARvec{u}}(\ARvec{x})
  +p_\alpha h_{\ARvec{u},\alpha}(\ARvec{x})
  \geq p_\alpha.
\]
Finally,
$
  \alpha^2/9\leq p_\alpha\leq\alpha^2/4
$
\end{lemma}

\begin{proof}
The common loss $\ell_{\ARvec{u}}$ is affine, and the rare loss $h_{\ARvec{u}, \alpha}$ is the maximum of two affine functions. Both are therefore continuous and convex. Since $\langle\ARvec{u},\ARvec{x}\rangle\in[-1,1]$, both take values in $[0,1]$. By definition, we can verify $h_{\ARvec{u},\alpha}(\ARvec{u})=1$, and $h_{\ARvec{u},\alpha}(\ARvec{v})=0$ whenever $\langle\ARvec{u},\ARvec{v}\rangle\leq\cos\alpha$.

It remains to prove the two inequalities. The definition of $p_\alpha$ gives
\[
  \frac{1-p_\alpha}{2}
  =\frac{p_\alpha}{1-\cos\alpha}
  =\frac{1}{3-\cos\alpha}.
\]
If $\langle\ARvec{u},\ARvec{x}\rangle\leq\cos\alpha$, then the rare loss is zero and the common-loss term is at least $p_\alpha$. Otherwise, the two terms sum to
\[
  \frac{1-\langle\ARvec{u},\ARvec{x}\rangle}{3-\cos\alpha}
  +\frac{\langle\ARvec{u},\ARvec{x}\rangle-\cos\alpha}{3-\cos\alpha}
  =p_\alpha.
\]
Finally, for $0<\alpha\leq1$, we have $
  \alpha^2/3\leq1-\cos\alpha\leq \alpha^2/2$ and $2\leq3-\cos\alpha\leq3$,
which proves the bounds on $p_\alpha$.
\end{proof}

\paragraph{Block construction.}
Fix a positive integer $K$ and set $\nu:=10^{-2}$. Start from $\ARvec{u}_1=(1,0)$ and, for $k=1,\ldots,K$, set
\[
  \alpha_k:=\frac{\sqrt\nu}{4^{k-1}},
  \quad
  p_k:=p_{\alpha_k},
  \quad
  n_k:=\left\lceil\frac{\nu}{p_k}\right\rceil.
\]
We construct the loss sequence as follows: for each block $k$ with $n_k$ rounds
\begin{itemize}
    \item[1.] We independently choose the rare loss $h_{\ARvec{u}_k,\alpha_k}$ with probability $p_k$ on each of the $n_k$ rounds, and choose the common loss $\ell_{\ARvec{u}_k}$ otherwise. 
    \item[2.] If no rare loss was sampled in block $k$, set $\ARvec{u}_{k+1} \leftarrow \ARvec{u}_k$. Otherwise at least one rare loss was sampled in block $k$, we update $\ARvec{u}_{k+1}$ by rotating $\ARvec{u}_k$ counterclockwise through $\alpha_k$. That is, if $\ARvec{u}_k = (\cos(\phi), \sin(\phi))$, then update $\ARvec{u}_{k+1} \leftarrow (\cos(\phi + \alpha_k), \sin(\phi + \alpha_k))$. 
\end{itemize}
Let $N_K = \sum_{k=1}^K n_k$. We note that this is an oblivious randomized adversary.  We use $\E_{\mathrm{loss}}$ for expectation over the construction and $\E_{\mathrm{loss},\mathcal A}$ for expectation over both the construction and the  randomness of $\mathcal A$.

\begin{lemma}[Cost of the algorithm]\label{lem:oco-lb-algorithm-cost}
For every online algorithm $\+A$,
\[
  \E_{\mathrm{loss},\mathcal A}\left[\sum_{t=1}^{N_K} f_t(\ARvec{x}_t)\right]
  \geq K\nu.
\]
\end{lemma}

\begin{proof}
Consider a round in block $k$ and condition on the past and on the private randomness used to choose $\ARvec{x}_t$. This fixes $\ARvec{x}_t$ and $\ARvec{u}_k$, while the current common-or-rare choice remains independent. By \cref{lem:oco-lb-gadget}, the conditional expected loss is at least $p_k$. Summing over the blocks gives
\[
  \E_{\mathrm{loss},\mathcal A}\left[\sum_{t=1}^{N_K} f_t(\ARvec{x}_t)\right]
  \geq\sum_{k=1}^K n_kp_k
  \geq K\nu. \qedhere
\]
\end{proof}

\begin{lemma}[Cost of the endpoint]\label{lem:oco-lb-comparator-cost}
The endpoint of the block construction $\ARvec{u}_\star:= \ARvec{u}_{K+1}$ satisfies
\[
  \E_{\mathrm{loss}}\left[\sum_{t=1}^{N_K} f_t(\ARvec{u}_\star)\right]
  \leq10K\nu^2.
\]
\end{lemma}

\begin{proof}
Let $I_j$ indicate whether block $j$ contains a rare loss. By the union bound and the definition of $n_j$,
\[
  \Pr(I_j=1)
  =1-(1-p_j)^{n_j}
  \leq n_jp_j
  \leq\nu+p_j
  <2\nu,
\]
where the last inequality uses $p_j\leq\alpha_j^2/4\leq\nu/4$ by \Cref{lem:oco-lb-gadget} and the definition $\alpha_j = \sqrt{\nu}/4^{j-1}$.

Define counterclockwise angle from $\ARvec{u}_k$ to $\ARvec{u}_\star$ as $
  \Delta_k:=\sum_{j=k}^K I_j\alpha_j$. By definition of $\alpha_k$, we have
\[
  0\leq\Delta_k\leq\sum_{j=k}^{\infty}\alpha_j
  =\frac{4\alpha_k}{3}<\pi.
\]
If block $k$ contains a rare loss, then $\Delta_k\geq\alpha_k$. Thus $\langle \ARvec{u}_{k}, \ARvec{u}_{\star}\rangle = \cos\Delta_k \le \cos \alpha_k$ and the rare loss $h_{\ARvec{u}_k, \alpha_k}(\ARvec{u}_\star) = 0$ is zero (\Cref{lem:oco-lb-gadget}). Each common loss in the block is at most $
  \ell_{\ARvec{u}_k}(\ARvec{u}_\star)
  =\frac{1-\cos\Delta_k}{2}
  \leq\frac{\Delta_k^2}{4}$.
Since $\Delta_k\leq4\alpha_k/3$, we have
\begin{align*}
  \E_{\mathrm{loss}}[\Delta_k^2] \leq\frac{4\alpha_k}{3}\E_{\mathrm{loss}}[\Delta_k] =\frac{4\alpha_k}{3}\sum_{j=k}^K\alpha_j\Pr(I_j=1)<\frac{4\alpha_k}{3}\cdot2\nu\cdot\frac{4\alpha_k}{3}
  <4\nu\alpha_k^2.
\end{align*}
Then by $p_k \ge \alpha_k^2/9$ (\cref{lem:oco-lb-gadget}) and $\alpha_k^2\leq\nu$, the expected loss of $\ARvec{u}_\star$ in block $k$ is at most
\[
  n_k \-E_{\mathrm{loss}}[ \ell_{\ARvec{u}_k}(\ARvec{u}_\star)]\le  \frac{n_k}{4} \-E_{\mathrm{loss}}[\Delta_k^2] \le n_k \nu \alpha_k^2 
  \leq\nu \left(\frac{\nu}{p_k}+1\right)\alpha_k^2
  \leq \nu (9\nu+\alpha_k^2)
  \leq10\nu^2.
\]
Summing over the blocks proves the claim.
\end{proof}

\begin{proof}[Proof of \Cref{thm:oco-lower}]
We assume $K:=\lfloor\log_{16}T\rfloor \ge 1$ since otherwise the claim is trivial. We use the loss construction described above. Since $\alpha_k^2=\nu/16^{k-1}$, \cref{lem:oco-lb-gadget} gives
$
  n_k
  \leq \nu/p_k+1 \le 9\nu/\alpha_k^2+1
  \leq9\cdot16^{k-1}+1.
$
Consequently,
\[
  N_K=\sum_{k=1}^K n_k\leq\frac{9(16^K-1)}{15}+K\leq16^K\leq T.
\]
We pad zero losses for rounds $N_K < t \le T$.

For every realization of the loss construction and of the algorithm's randomness, we evaluate alternating regret at $\ARvec{u}_\star$. Since all losses are nonnegative and each comparator loss appears at most twice, we have
\[
  \AltReg^T
  \geq\sum_{t=1}^{N_K} f_t(\ARvec{x}_t)
  -2\sum_{t=1}^{N_K} f_t(\ARvec{u}_\star).
\]
Applying \cref{lem:oco-lb-algorithm-cost,lem:oco-lb-comparator-cost} and using $\nu=10^{-2}$ gives
\[
  \E_{\mathrm{loss}}\left[
    \E_{\mathcal A}\left[\AltReg^T\mid f_1,\ldots,f_T\right]
  \right] =\E_{\mathrm{loss},\mathcal A}[\AltReg^T]
  \geq K\nu-20K\nu^2
  =\frac{K}{125}.
\]
Thus some fixed realization of the loss sequence satisfies the desired bound. 
\end{proof}

\subsection{An $\Omega(d\log (1+T/d))$ lower bound for OCO}\label{sec:oco-lower-dim}

\Cref{thm:oco-lower} shows that the $\log T$ factor in \Cref{cor:pointwise} is
unavoidable even for $2$-dimensional unit ball. We now show that the linear dependence on $d$ is also necessary in
the worst case over a convex set $\+X \subseteq B_d(1)$ and prove the optimal lower bound of $\Omega(d\log (1+T/d))$.

\paragraph{High-level idea:} Roughly speaking, we construct the losses in phases,  where in each phase we simulate the lower bound construction in \Cref{thm:oco-lower} on $2$ coordinates. There are $d/2$ phases and each phase contains $m = \Theta(T/d)$ rounds. By \Cref{thm:oco-lower}, the alternating regret in each phase is at least $\Omega(\log(1 + T/d))$. Thus the total alternating regret is at least $\Omega(d \log (1+ T/d))$. 

Let $r:=1/\sqrt d$. Define
\[
  \+X
  :=
  \underbrace{rB_2(1)\times\cdots\times rB_2(1)}_{\lfloor d/2\rfloor \ \text{times}}
  \times[-r,r]^{\,d-2\lfloor d/2\rfloor}
  \ \subseteq\ \R^{d},
\]
where $d-2\lfloor d/2\rfloor\in\{0,1\}$. The set $\+X$ is compact, convex, and
full-dimensional. Moreover, every $\ARvec{x}\in\+X$ satisfies
\[
  \lVert\ARvec{x}\rVert_2^2
  \le \lfloor d/2\rfloor r^2+(d-2 \lfloor d/2\rfloor)r^2
  =(d-\lfloor d/2\rfloor)r^2
  \le1,
\]
so $\+X\subseteq B_d(1)$. For $j\in[\lfloor d/2\rfloor]$, let $\pi_j:\+X\to B_2(1)$ be the
projection onto the $j$-th two-dimensional block, rescaled by $1/r$. Thus
$\pi_j$ is linear and $\pi_j(\+X)=B_2(1)$. Throughout this subsection, let
$\AltReg^T:=\max_{\ARvec{u}\in\+X}\AltReg^T(\ARvec{u})$.

\begin{theorem}[Linear dimension dependence is necessary]\label{thm:oco-lower-dim}
For every $d\ge2$, every horizon $T\ge16$, and every possibly randomized online
algorithm $\mathcal A$, there is a sequence of continuous convex losses
$f_1,\ldots,f_T:\+X\to[0,1]$ such that
\[
  \E_{\mathcal A}\left[\AltReg^T\right]\ \ge\ \frac{mK}{125}
  \quad\text{where}\quad
  m:=\min\left\{\left\lfloor\frac d2\right\rfloor,\left\lfloor\frac T{16}\right\rfloor\right\},
  \quad
  K:=\left\lfloor\log_{16}\left\lfloor\frac Tm\right\rfloor\right\rfloor .
\]
In particular $\E_{\mathcal A}[\AltReg^T]=\Omega\!\left(d\log(1+T/d)\right)$.
\end{theorem}

\begin{proof}
Set $T':=\lfloor T/m\rfloor$. Since $m\le\lfloor T/16\rfloor$, we have
$T'\ge16$ and $K=\lfloor\log_{16}T'\rfloor\ge1$. Since $m\le \lfloor d/2\rfloor$, we use
the first $mT'$ rounds as $m$ consecutive phases of $T'$ rounds, one for each
of the first $m$ two-dimensional blocks.

We construct the losses in these phases successively. For $j\in[m]$, suppose
the losses in the first $j-1$ phases have been fixed. This fixed prefix induces
a distribution over $\mathcal A$'s internal state. Starting from this state,
simulate $\mathcal A$ for $T'$ more rounds: return its decisions projected by
$\pi_j$, and feed each two-dimensional loss back to $\mathcal A$ after lifting
it through $\pi_j$. This defines a possibly randomized online algorithm on
$B_2(1)$. Inspecting the proof of \Cref{thm:oco-lower}, a fixed sequence of
continuous convex losses
$\tilde f^{(j)}_1,\ldots,\tilde f^{(j)}_{T'}:B_2(1)\to[0,1]$ can be chosen for
this projected algorithm. Let
$\ARvec{u}^{(j)}_\star\in B_2(1)$ minimize
$\sum_{i=1}^{T'}\tilde f^{(j)}_i(\ARvec{u})$. Then
\[
  \E_{\mathcal A}\left[
    \sum_{i=1}^{T'}
    \tilde f^{(j)}_i\!\left(
      \pi_j\!\left(\ARvec{x}_{(j-1)T'+i}\right)
    \right)
    -2\sum_{i=1}^{T'}\tilde f^{(j)}_i(\ARvec{u}^{(j)}_\star)
  \right]
  \ge\frac{K}{125}.
\]
For $t=(j-1)T'+i$, set $f_t:=\tilde f^{(j)}_i\circ\pi_j$. After all $m$
phases have been fixed, set $f_t\equiv0$ for $mT'<t\le T$. Each phase is fixed
for the randomized algorithm induced by the preceding phases, not for a
realization of $\mathcal A$'s random choices. Thus the resulting loss sequence
is deterministic and oblivious. Since $\pi_j$ is linear, every $f_t$ is
continuous, convex, and takes values in $[0,1]$.

Let $\ARvec{u}_\star\in\+X$ have $j$-th two-dimensional block
$r\ARvec{u}^{(j)}_\star$ for $j\in[m]$ and all remaining coordinates equal to
zero. Then $\pi_j(\ARvec{u}_\star)=\ARvec{u}^{(j)}_\star$ for every $j\in[m]$.
Since the losses are nonnegative, we may discard the terms
$f_{t-1}(\ARvec{x}_t)$ from the learner's loss, while every comparator loss
appears at most twice. Therefore,
\begin{align*}
  \AltReg^T
  \ge
  \sum_{t=1}^T f_t(\ARvec{x}_t)
  -2\sum_{t=1}^T f_t(\ARvec{u}_\star)=\sum_{j=1}^m\left[
    \sum_{i=1}^{T'}
    \tilde f^{(j)}_i\!\left(
      \pi_j\!\left(\ARvec{x}_{(j-1)T'+i}\right)
    \right)
    -2\sum_{i=1}^{T'}\tilde f^{(j)}_i(\ARvec{u}^{(j)}_\star)
  \right].
\end{align*}
Taking expectation and summing the phase guarantees gives
\[
  \E_{\mathcal A}[\AltReg^T]\ge\frac{mK}{125}.
\]

It remains to verify the asymptotic bound. If $T\ge8d$, then
$m=\lfloor d/2\rfloor\ge d/3$ and
$
  T'=\lfloor T/m \rfloor
  \ge T/(2m)
  \ge T/d.
$
Since $T'\ge16$, we have
$K=\Omega(\log(1+T'))=\Omega(\log(1+T/d))$. Hence
$mK=\Omega(d\log(1+T/d))$. If $T<8d$, then $m=\lfloor T/16\rfloor\ge T/32$ and $K\ge1$. Since $\log(1+x)\le x$, we have $ mK\ge T/32\ge d\log\left(1+ T/d\right)/32$. Thus we conclude $\E_{\mathcal A}[\AltReg^T] = \Omega(d\log(1+T/d))$.
\end{proof}

Combining \Cref{thm:oco-lower-dim} with \Cref{cor:pointwise} pins down the worst-case alternating regret over
$d$-dimensional bodies in every regime of $T$ and $d$.

\begin{corollary}[Minimax-optimal alternating rerget for OCO]\label{cor:oco-tight}
Let $\+R^\star(d,T):=\sup_{\+X}\inf_{\+A}\sup_{f_1,\ldots,f_T}\AltReg^T$, where
the outer supremum is over compact convex $\+X\subseteq\R^d$ and the losses are
convex with values in $[-1,1]$. Then, for all $d\ge2$ and $T\ge16$,
\[
  \+R^\star(d,T)=\Theta\!\left(d\log\left(1+\frac Td\right)\right).
\]
\end{corollary}


\section{Conclusion}
In this paper, we settle the minimax-optimal alternating regret for both OLO and OCO. For OLO over the simplex $\Delta_d$, we give the alternation-aware Hedge (AA-Hedge) algorithm with optimal $O(\log d)$ regret. For OCO, we give the continuous AA-Hedge algorithm with $O(d \log (1+T/d))$ regret and proves a $\Omega(d \log (1+T/d))$ lower bound. Our results imply the first alternating learning dynamics with $O(1/T)$ convergence to CCE in two-player general-sum games. 

\section*{Acknowledgement} The proofs were first obtained with the assistance of ChatGPT 5.6 Sol and Claude Opus 5. The authors subsequently verified and significantly rewrote the proofs for better exposition.

\bibliographystyle{abbrvnat}
\bibliography{refs}

\appendix
\section{Proofs for Online Convex Optimization}\label{sec:oco-proofs}
The proof follows the same strategy as in \Cref{sec:simplex}: we first show that the update is well defined and feasible, and then use a one-step inequality to prove that a potential function is nonincreasing.

\subsection{Existence and feasibility of the update}

\begin{proof}[Proof of \Cref{lem:continuous-center}]
Define
\[
F_t(\theta):=\log\int_{\+X}
\exp\!\left[-\eta L_{t-1}(\ARvec{u})-\beta(f_{t-1}(\ARvec{u})-\theta)^2\right]
\mu(\dd\ARvec{u}),
\]
and let $P_{t,\theta}$ be the probability measure obtained by normalizing the integrand. $F_t$ is differentiable since it is logarithm of a bounded integrand over a compact set. We calculate its first-order and second-order derivative:
\begin{align*}
F_t'(\theta)
&=2\beta\left(\E_{P_{t,\theta}}[f_{t-1}]-\theta\right),\\
F_t''(\theta)
&=4\beta^2\operatorname{Var}_{P_{t,\theta}}(f_{t-1})-2\beta.
\end{align*}
Since $f_{t-1}\in[-1,1]$, the variance is at most $1$. Since $\beta<1/2$, we have $F_t''(\theta)\le4\beta^2-2\beta<0$, so $F_t$ is strictly concave. Moreover, $F_t(\theta)\to-\infty$ as $|\theta|\to\infty$, so the maximizer exists and is unique. The first-order condition gives
\[
\theta_t=\E_{P_t}[f_{t-1}]
\in\left[\min_{\ARvec{u}\in\supp(\mu)}f_{t-1}(\ARvec{u}),
\max_{\ARvec{u}\in\supp(\mu)}f_{t-1}(\ARvec{u})\right].
\]
Thus $\E_{P_t}[\xi_t]=0$ and $\xi_t(\ARvec{u})\in[-2,2]$ on $\supp(P_t)$. It follows that $1-\xi_t(\ARvec{u})/2\ge0$ and
\[
Q_t(\+X)=1-\frac12\E_{P_t}[\xi_t]=1.
\]
Hence $Q_t$ is a probability measure. Since $\+X$ is compact and convex, the barycenter of $Q_t$, $\ARvec{x}_t$, lies in $\+X$.
\end{proof}

\subsection{A potential proof of logarithmic alternating regret}

Define the learner's cumulative effective loss through round $t-1$ by $A_{t-1}:=\sum_{s=1}^{t-1}g_s(\ARvec{x}_s)$, and define its pointwise regret by $R_{t-1}(\ARvec{u}):=A_{t-1}-L_{t-1}(\ARvec{u})$.  Consider the potential
\begin{equation}\label{eq:oco-potential}
  \Phi_t
  :=\max_{\theta\in\R}
  \log
  \int_{\X}
  \exp\!\left(
    \eta R_{t-1}(\ARvec{u})-\beta(f_{t-1}(\ARvec{u})-\theta)^2
  \right)\mu(\dd \ARvec{u}).
\end{equation}
The factor $e^{\eta A_{t-1}}$ is common to all comparator points, so the maximizing centering parameter and normalized distribution in \cref{eq:oco-potential} are exactly $\theta_t$ and $P_t$ in \Cref{alg:main}. Initially, $\Phi_1=0$.

\begin{theorem}[Monotonicity of potential]\label{thm:continuous-potential}
Assume $\eta\le1/8$ and $\beta=\eta/4$. For every round $t\ge1$, we have $\Phi_{t+1}\le\Phi_t\le0$.
\end{theorem}

\begin{proof}
Define $\xi_t(\ARvec{u}):=f_{t-1}(\ARvec{u})-\theta_t$ and $\xi_{t+1}(\ARvec{u}):=f_t(\ARvec{u})-\theta_{t+1}$. By \Cref{lem:continuous-center}, both functions take values in $[-2,2]$, $\E_{P_t}[\xi_t]=0$, $Q_t=(1-\xi_t/2)P_t$, and $\ARvec{x}_t=\E_{Q_t}[\ARvec{u}]$.

Since $g_t=f_{t-1}+f_t$ is convex, Jensen's inequality gives $g_t(\ARvec{x}_t)\le\E_{Q_t}[g_t]$. The two scalar centers cancel, so for every $\ARvec{u}\in\+X$,
\[
R_t(\ARvec{u})-R_{t-1}(\ARvec{u})
\le \E_{Q_t}[\xi_t+\xi_{t+1}]-\xi_t(\ARvec{u})-\xi_{t+1}(\ARvec{u}).
\]
By the definition of the potential and the fact that $P_t$ is its normalized exponential-weight distribution, we have
\begin{align*}
\Phi_{t+1}-\Phi_t
&=\log\E_{P_t}\exp\!\left\{
\eta(R_t-R_{t-1})+\beta(\xi_t^2-\xi_{t+1}^2)
\right\}\\
&\le\log\E_{P_t}\exp\!\left\{
\eta\bigl(\E_{Q_t}[\xi_t+\xi_{t+1}]-\xi_t-\xi_{t+1}\bigr)
+\beta(\xi_t^2-\xi_{t+1}^2)
\right\}.
\end{align*}
We use the following continuous counterpart of \Cref{lem:one-step}.

\begin{lemma}[Functional one-step inequality]\label[lemma]{lem:functional}
Let $P$ be a probability measure on an arbitrary measurable space. Let $\xi,\zeta$ be measurable functions with values in $[-2,2]$ such that $\E_P[\xi]=0$, and define $Q(\dd u):=(1-\xi(u)/2)P(\dd u)$. If $\eta\le1/8$ and $\beta=\eta/4$, then
\[
\log\E_P\exp\!\left\{
\eta\bigl(\E_Q[\xi+\zeta]-\xi-\zeta\bigr)
+\beta(\xi^2-\zeta^2)
\right\}\le0.
\]
\end{lemma}

\begin{proof}
For any measurable function $z$ with values in $[-2,2]$, define
\[
H(z):=\log\E_P\exp\!\left\{
\eta\bigl(\E_Q[\xi+z]-\xi-z\bigr)+\beta(\xi^2-z^2)
\right\}.
\]
It suffices to show that $H(\zeta)\le0$. Let $\Pi_z$ be the probability measure obtained by normalizing the exponential in the definition of $H(z)$. To show that $H$ is concave, fix functions $z,w$ such that $z_s:=z+sw$ remains in $[-2,2]$, and write $v_s:=-\eta-2\beta z_s$. $H(z_s)$ is differentiable since it is logarithm of a bounded integrand over a compact set. We note that $\-E_{Q}[\xi + z_s]$ is affine and contributes noting to the second-order derivative.
\[
\frac{\dd^2}{\dd s^2}H(z_s)
=\E_{\Pi_{z_s}}\!\left[(v_s^2-2\beta)w^2\right]
-\left(\E_{\Pi_{z_s}}[v_sw]\right)^2.
\]
Since $z_s\in[-2,2]$, $\beta=\eta/4$, and $\eta\le1/8$, we have $v_s\in[-2\eta,0]$ and $v_s^2\le4\eta^2\le2\beta$. Thus $H$ is concave.

At $z=-\xi$, every exponent vanishes, so $H(-\xi)=0$ and $\Pi_{-\xi}=P$. For every bounded direction $w$, the directional derivative is
\begin{align*}
D H(-\xi)[w]
&=\eta\E_Q[w]+\E_P[(-\eta+2\beta\xi)w]\\
&=(2\beta-\eta/2)\E_P[\xi w]=0,
\end{align*}
where the second equality uses $Q=(1-\xi/2)P$ and the last uses $\beta=\eta/4$. Taking $w=\zeta+\xi$, the path $-\xi+sw$ remains in $[-2,2]$ for $s\in[0,1]$. Concavity therefore gives $H(\zeta)\le H(-\xi)=0$.
\end{proof}

Applying \Cref{lem:functional} with $P=P_t$, $\xi=\xi_t$, and $\zeta=\xi_{t+1}$ gives $\Phi_{t+1}\le\Phi_t$. Since $\Phi_1=0$, the claim follows.
\end{proof}

\subsection{\texorpdfstring{Proof of \Cref{thm:distributional}}{Proof of the distributional alternating-regret bound}}\label{sec:proof-distributional}

\begin{proof}[Proof of \Cref{thm:distributional}]
The choices $\eta=1/8$ and $\beta=1/32$ in \Cref{alg:main} satisfy the conditions in \Cref{thm:continuous-potential}. Thus $\Phi_{T+1}\le\Phi_1=0$. By the definition of the potential, for every $\theta\in\R$,
\[
  \log\int_{\X}
  \exp\!\left(
    \eta R_T(\ARvec{u})-\beta(f_T(\ARvec{u})-\theta)^2
  \right)\mu(\dd \ARvec{u})
  \le0.
\]
By \cref{eq:alt-regret}, $R_T(\ARvec{u})=\AltReg^T(\ARvec{u})$.

If $\KL(\rho\|\mu)=\infty$, the claim is immediate. Otherwise, write
$r:=\dd\rho/\dd\mu$. Since $r>0$ $\rho$-almost surely, for every $\theta\in\R$,
\begin{align*}
&\int_{\X}
\exp\!\left(\eta R_T-\beta(f_T-\theta)^2\right)\dd\mu \\
&\qquad\ge
\int_{\{x \in \+X:r>0\}}
\exp\!\left(\eta R_T-\beta(f_T-\theta)^2\right)\dd\mu \\
&\qquad=
\int_{\X}
\exp\!\left(\eta R_T-\beta(f_T-\theta)^2-\log r\right)\dd\rho.
\end{align*}
Here the inequality discards the possible $\mu$-mass on $\{x \in \+X: r=0\}$, and the equality uses $\dd\rho=r\,\dd\mu$.
Taking logarithms and applying Jensen's inequality to the probability measure $\rho$ gives
\begin{align*}
0
&\ge\log\int_{\X}
\exp\!\left(\eta R_T-\beta(f_T-\theta)^2\right)\dd\mu \\
&\ge\log\int_{\X}
\exp\!\left(\eta R_T-\beta(f_T-\theta)^2-\log r\right)\dd\rho \\
&\ge\E_\rho\!\left[\eta R_T-\beta(f_T-\theta)^2-\log r\right] \\
&=\eta\AltReg^T(\rho)-\beta\E_\rho[(f_T-\theta)^2]-\KL(\rho\|\mu).
\end{align*}
Choose $\theta=\E_\rho[f_T]$. Since $1/\eta=8$ and $\beta/\eta=1/4$, rearranging proves the first inequality in \cref{eq:distributional-bound}. The second follows from $\operatorname{Var}_\rho(f_T)\le1$. 
\end{proof}

\end{document}